\pdfoutput=1

\documentclass[11pt]{article}

\usepackage[final]{acl}

\usepackage{times}
\usepackage{latexsym}

\usepackage[T1]{fontenc}

\usepackage[utf8]{inputenc}

\usepackage{microtype}

\usepackage{inconsolata}

\usepackage{hhline} 
\usepackage{cellspace} 

\usepackage{arydshln}
\usepackage{graphicx}
\usepackage{booktabs}
\usepackage{multirow}
\usepackage{amsthm}
\usepackage{enumitem}
\usepackage{amsmath,amsfonts}
\usepackage{subfigure}
\usepackage{bm}

\newtheorem{lemma}{Lemma}
\newtheorem{proposition}{Proposition}
\newtheorem{myexample}{Example}
\newcommand{\name}{\textsf{RoseLoRA}}
\usepackage{setspace}
\AtBeginDocument{%
  \addtolength\abovedisplayskip{-0.05\baselineskip}%
  \addtolength\belowdisplayskip{-0.05\baselineskip}%
  \addtolength\abovedisplayshortskip{-0.05\baselineskip}%
  \addtolength\belowdisplayshortskip{-0.05\baselineskip}%
}
\usepackage{titlesec}
\titlespacing*{\section}{0pt}{0.25\baselineskip}{0.25\baselineskip}
\titlespacing*{\subsection}{0pt}{0.25\baselineskip}{0.25\baselineskip}
\titlespacing*{\subsubsection}{0pt}{0.25\baselineskip}{0.25\baselineskip}
\title{{\name}: Row and Column-wise Sparse Low-rank Adaptation of Pre-trained Language Model for Knowledge Editing and Fine-tuning}

\author{Haoyu Wang$^\dagger{}$, Tianci Liu$^{\S}$, Ruirui Li$^\star{}$, Monica Xiao Cheng$^\star{}$, Tuo Zhao$^*$, and Jing Gao$^{\S}$ \\
  $^\dagger{}$SUNY Albany, Albany, NY, USA\\
  $^{\S}$Purdue University, West Lafayette, IN, USA\\
  $^*$Georgia Institute of Technology, Atlanta, GA, USA\\
  $^\star{}$Amazon, Palo Alto, CA, USA\\
  \texttt{$^\dagger{}$hwang28@albany.edu},
  \texttt{$^{\S}$\{liu3351,jinggao\}@purdue.edu}, \\\texttt{$^*$tourzhao@gatech.edu} , \texttt{$^\star{}$\{ruirul,chengxca\}@amazon.com}}

\definecolor{amaranth}{rgb}{0.9, 0.17, 0.31}

\begin{document}
\maketitle
\begin{abstract}
Pre-trained language models, trained on large-scale corpora, demonstrate strong generalizability across various NLP tasks. Fine-tuning these models for specific tasks typically involves updating all parameters, which is resource-intensive. Parameter-efficient fine-tuning (PEFT) methods, such as the popular LoRA family, introduce low-rank matrices to learn only a few parameters efficiently. However, during inference, the product of these matrices updates all pre-trained parameters, complicating tasks like knowledge editing that require selective updates. We propose a novel PEFT method, which conducts \textbf{r}ow and c\textbf{o}lumn-wise spar\textbf{se} \textbf{lo}w-\textbf{r}ank \textbf{a}daptation ({\name}), to address this challenge. {\name} identifies and updates only the most important parameters for a specific task, maintaining efficiency while preserving other model knowledge. By adding a sparsity constraint on the product of low-rank matrices and converting it to row and column-wise sparsity, we ensure efficient and precise model updates. Our theoretical analysis guarantees the lower bound of the sparsity with respective to the matrix product.
Extensive experiments on five benchmarks across twenty datasets demonstrate that {\name} outperforms baselines in both general fine-tuning and knowledge editing tasks.

\end{abstract}

\section{Introduction}
Pre-trained language models, trained on extensive and diverse general-domain corpora, exhibit robust generalization capabilities, benefiting various natural language processing (NLP) tasks, such as natural language understanding~\cite{kenton2019bert,liu2019roberta} and generation~\cite{touvron2023llama,ouyang2022training}. To further adapt these pre-trained models to a specific downstream task, fine-tuning is typically performed. 
However, these models often comprise numerous parameters, rendering full fine-tuning  
resource-intensive.

To address this challenge, parameter-efficient fine-tuning~(PEFT) methods~\cite{ding2023parameter,han2024parameter} are proposed. These method introduce a small number of learnable parameters and update only the lightweight introduced parameters during fine-tuning.
Among existing methods, LoRA family~\cite{hu2021lora,zhang2023adaptive,ding2023parameter,liu2024dora} has gained remarkable popularity because of its high efficiency and good performance. Conceptually, these LoRA methods add new low-rank matrices to model weights for fine-tuning.
Unlike other PEFT methods such as Adapter \citep{houlsby2019parameter}, LoRA family does not modify the model architecture and is easier to incorporate.

LoRA family has demonstrated notable performance on tasks, such as commonsense reasoning and arithmetic reasoning~\cite{hu2023llm,liu2024dora}, that mainly rely on a language model's ability to understand and generate text without requiring to modify its internal knowledge explicitly. 
However, some specialized tasks require updating this internal knowledge. 
For instance, in knowledge editing~\cite{zhang2024comprehensive,de2021editing}, a language model should incorporate new provided knowledge while preserving other existing knowledge simultaneously.  
On such tasks, the LoRA family of methods are less-suited due to the coarse-grained control they offer.
In particular, 
the product of the low-rank matrices introduced by LoRA methods is a dense matrix, which is added to the pre-trained model weights during inference. Consequently, all pre-trained parameters are updated, making it challenging to selectively modify specific internal knowledge. 
This motivates a natural question: 
\textbf{\textit{Is there a PEFT method that can be effectively employed for tasks that require editing the internal knowledge of language models?}}

To answer this question, we propose a \textbf{r}ow and c\textbf{o}lumn-wise spar\textbf{se} \textbf{lo}w-\textbf{r}ank \textbf{a}daptation method~({\name}). 
The motivation is to identify and update only the most important and influential parameters in the pre-trained model concerning a specific task. 
In this way, the pre-trained model can be updated effectively with minimal impacts on knowledge that does not require modification.
Specifically, {\name} inherits the structure of LoRA to enable parameter-efficient fine-tuning.
To selectively fine-tune the most important parameters, we introduce a sparsity constraint, i.e., the $\ell_0$ norm, on the product of the low-rank matrices. 
However, this constraint is non-trivial to optimize.
While $\ell_{0}$ norm constraint is widely explored in model pruning~\citep{zhu2017prune,wang2019structured,sun2023simple}, these methods can only address the sparsity constraint on each low-rank matrix individually. 
Unfortunately, even if each low-rank matrix is sparse, this does not guarantee that their product will be sparse. To overcome this challenge, we propose converting the original sparsity constraint to row and column-wise sparsity constraints on two low-rank matrices (i.e., $\bm{B}$ and $\bm{A}$ in LoRA). We provide a theoretical lower bound of the sparsity of the product of the two low-rank matrices. Furthermore, we propose using a sensitivity-based importance score to incrementally solve the row and column-wise sparsity constraints.

Beyond knowledge editing, the proposed {\name} can also be applied to other general tasks, e.g., commonsense and arithmetic reasoning, instruction following, and natural language understanding. {\name} updates the few most important parameters of the model via enforcing the row or column-wise sparsity for the low-rank matrices, and can match or even outperform LoRA performance with significantly fewer modified parameters.

The contributions are summarized as follows: 1)~We propose {\name}, a novel PEFT method that detects and optimizes the most important task-related parameters, resulting in highly precise and effective model updates while being more lightweight than existing methods. 2)~We propose a novel row and column-wise sparsity constraint to control the sparsity of the product of two low-rank matrices. Additionally, we provide a theoretical sparsity lower bound for the proposed {\name}. 3)~We conduct extensive experiments on five benchmarks covering over twenty datasets. The experiments show that the proposed {\name} can outperform baselines on both general fine-tuning tasks and knowledge editing tasks.

\section{Related Works}
\vspace{-0.05in}
In this section we provide a concise overview of related works. 
\subsection{Parameter Efficient Fine-Tuning (PEFT)}
PEFT injects a small fraction of trainable parameters into pre-trained large language models (LLMs) to adapt them to downstream tasks.
Prefix Tuning~\citep{li2021prefix} prepends soft tokens to the input and learns their continuous embeddings while keeping the original parameters frozen. Adapter~\citep{houlsby2019parameter,he2021effectiveness}, on the other hand, inserts lightweight bottleneck neural network modules into the transformer blocks.
The third paradigm, LoRA and its variants \citep{hu2021lora, zhang2023adaptive, ding2023sparse,dettmers2024qlora,li2023loftq,liu2024dora}, learns low-rank matrices to approximate the desired updates of the original model weights and has achieved state-of-the-art performance. 
Recently, ReFT \citep{wu2024reft} learns low-rank updates on model representations instead of weights and achieves performance comparable to LoRA with significantly fewer parameters. However, the underlying linear representation hypothesis may not hold valid \citep{engels2024not}, which greatly undermines its generalization ability.
In this work, we propose an effective method to learn sparse and low-rank updates on model weights, demonstrating superior performance using as few parameters as ReFT. Recent works such as AdaLoRA \citep{zhang2023adaptive} and SoRA \citep{ding2023sparse} have applied pruning to LoRA to increase its computational efficiency. 
However, it is worth mentioning that the proposed {\name} is significantly different from these methods. 
In particular, these works prunes to control the rank of learned model updates, but the updates are still dense in the sense that all parameters are affected, and cannot offer precise updates as {\name} thereof.

\begin{figure*}[htb!] 
\centering        
\includegraphics[width=0.95\textwidth]{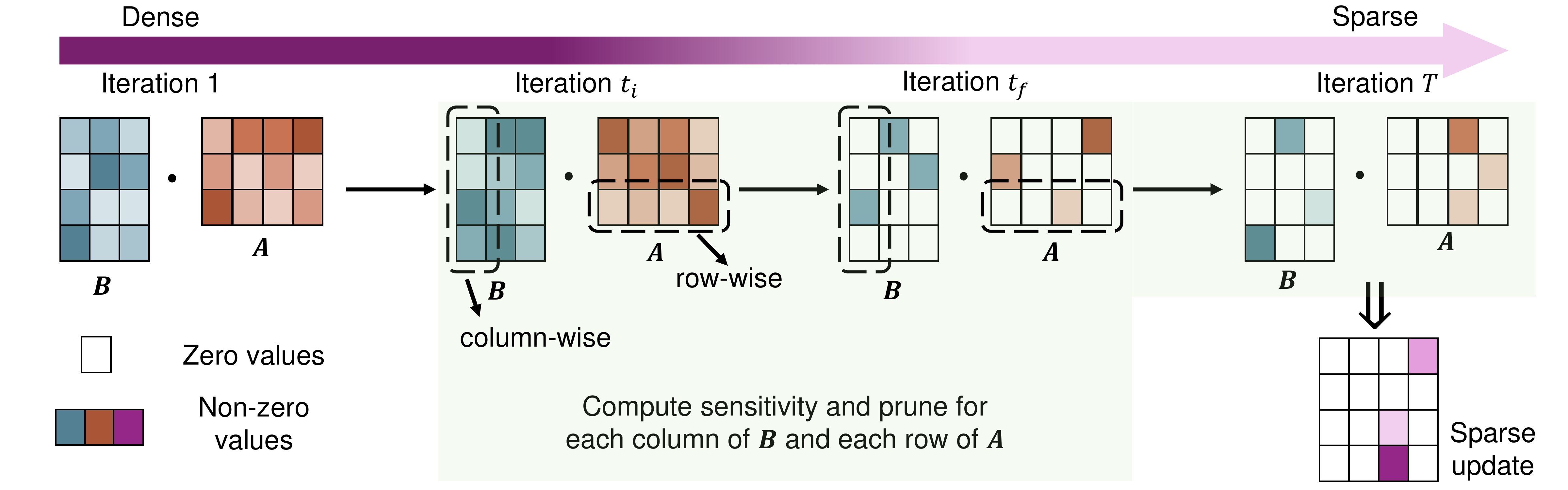} 
\caption{The framework of proposed \name.} 
\label{fig:framework}
\end{figure*}

\subsection{Knowledge Editing}
Knowledge editing seeks to update outdated knowledge in pre-trained LLMs to accommodate a dynamic world. Early efforts involved fine-tuning their parameters directly but suffered from severe forgetting of original knowledge~\cite{wang2023knowledge}. For more precise editing, only a minimal amount of parameters should be updated \citep{wang2023knowledge}. This requires sparse parameter updates, which proves NP-hard to solve \cite{natarajan1995sparse}. As a workaround, \citet{zhu2020modifying} used a relaxed $L_2$ norm constraint on the updates, and \citet{huang2023transformer, dong2022calibrating} limited the updates to feed-forward network (FFN) layers based on findings that learned knowledge is often stored in these layers \citep{dai2021knowledge}. For further refinement, the locate-and-edit paradigm \citep{meng2022locating, meng2022mass} identifies the layer storing specific knowledge and then modifies its parameters. Nonetheless, \citep{hase2024does} found that updating parameters other than the located ones can also achieve competitive editing performance, questioning the extent to which the more computationally expensive locating process benefits editing. 

Alternative solutions restore to external memory without updating original parameters, such as MEND~\cite{mitchell2021fast}, IKE~\cite{zheng2023can}, and SERAC~\cite{mitchell2022memory}. However, these methods require hard-to-access data to retrieve from (e.g., IKE) or to train extra models on (e.g., MEND and SERAC), which limits their practicality.
Recently, LoRA has also been applied for knowledge editing \citep{wu2023eva}. 
However, they do not provide the aforementioned sparsity guarantee, which will be discussed shortly in the next section, so they are less effective and show unsatisfactory performance~\citep{zhang2024comprehensive}.

\section{Preliminary}
In this section, we first briefly introduce the low-rank adaptation~(LoRA) and then introduce importance-aware pruning. 
\subsection{Low-rank Adaptation}
The LoRA models the efficient incremental update of pre-trained language models via the product of two learnable low-rank matrices. Specifically, the modified weight $\bm{W}$ can be represented as
\begin{align}
    \bm{W}=\bm{W}^{o}+\Delta=\bm{W}^{o}+\bm{B}\bm{A},
\end{align}
where $\bm{W}^{o}, \Delta\in\mathbb{R}^{d_{1}\times d_{2}}$ are the pre-trained weight matrix and the updated matrix respectively, $\bm{A}\in \mathbb{R}^{r\times d_{2}}$ and $\bm{B}\in \mathbb{R}^{d_{1}\times r}$ with $r\ll \min\{d_{1},d_{2}\}$. During fine-tuning, the pre-trained weight $\bm{W}^{o}$ is frozen and only lightweight matrices $\bm{A}$ and $\bm{B}$ will be updated, which can be formulated as 
\begin{align}
    \min_{\bm{A},\bm{B}}\ \mathcal{L}(\mathcal{D};\bm{W}^{o}+\bm{B}\bm{A}),
\end{align}
where $\mathcal{D}$ is the training dataset.

\subsection{Sensitivity-based Importance Score for Pruning}
Importance-aware pruning~\cite{sanh2020movement, han2015learning, molchanov2019importance, zhang2022platon, li2023losparse} aims to identify and set redundant model weights to zero based on estimated importance scores. Parameters with high importance scores are retained, while others are set to zero. Sensitivity~\cite{sanh2020movement, molchanov2019importance, li2023losparse} is a popular importance metric that measures the approximate change in training loss when setting a parameter to zero. Formally, the sensitivity with respect to weight $\bm{W}_{ij}$ is defined by the product of the weight and its corresponding gradient:
\begin{align}
    I(\bm{W}_{ij})=|\bm{W}_{ij}\cdot \nabla_{W_{ij}}\mathcal{L}|.
\end{align}
We denote the sensitivity at the $t$-th iteration based on the current mini-batch as $I^{(t)}$. To reduce the variance of sensitivity, \citet{zhang2022platon} proposed to apply exponential moving average for smoothing:
\begin{align}
\label{eq:sensitivity}
    \bar{I}^{(t)}(\bm{W}_{ij})=\beta \bar{I}^{(t-1)}(\bm{W}_{ij})+(1-\beta)I^{(t)},
\end{align}
where $\beta$ is a hyper-parameter.

\section{Methodology}
\label{sec:method}
To efficiently fine-tune a pre-trained language model with selective updating, we propose \name, a novel LoRA-style fine-tuning framework with sparse adaptation. The framework is illustrated in Figure~\ref{fig:framework}. We introduce row and column-wise sparsity constraints on the two low-rank matrices, respectively. We theoretically prove that the sparsity lower bound of the product of these low-rank matrices can be guaranteed under these constraints.
\subsection{Row and Column-wise Sparse Low-rank Adaptation}
We aim to update minimal parameters to enable the model to fit the training data, retain more previous knowledge, and become more lightweight. To achieve this goal, we build on the popular and effective parameter-efficient fine-tuning method LoRA, resulting in the following loss function:
\begin{align}
\label{eq:loss}
    \notag\min_{\bm{A},\bm{B}} & \quad \mathcal{L}(\mathcal{D};\bm{W}^{o}+\bm{B}\bm{A})\\
    \text{s.t.} &\quad \frac{\|\bm{B}\bm{A}\|_{0}}{d_{1}d_{2}}\leq \tau,
\end{align}
where $\tau$ is the sparsity threshold. However, Eqn.~\ref{eq:loss} is challenging to handle, with difficulty lie in two-fold.
First, the $\ell_{0}$ optimization is NP-hard.
Despite that some effective approximate solutions have been proposed \citep{zhu2017prune, wang2019structured, sun2023simple}, 
they cannot be applied directly.
In particular, due to the complex product-based parameterization, 
it is extremely hard to learn parameters in $\bm A, \bm B$ even if we know which entries in their product $\bm B \bm A$ should be 0. Furthermore, simply controlling the sparsity of $\bm{B}$ and $\bm{A}$ may not work, as shown in Example~\ref{exp:example}.

\begin{myexample}
\label{exp:example}
Let $s(\cdot)$ represent the sparsity (i.e., the portion of zero entries) of a vector or matrix. For sparse matrix $\bm{A}=[\bm{a}^\top;\bm{0}^{(r-1)\times d_{2}}]$ and $\bm{B}=[\bm{b},\bm{0}^{d_{1}\times(r-1)}]$, where $\bm{a}$ and $\bm{b}$ contains non-zero entries, we have $s(\bm A) = s(\bm B) = \frac{r-1}{r}$ that is reasonably large for $r > 1$. However, $s(\bm B \bm A) = s(\bm b \bm a^\top) = 0$, i.e., the product is a dense matrix.
\end{myexample}

To summarize, it is non-trivial to incorporate sparsity in LoRA. To address this challenge, we propose controlling the sparsity of each row of $\bm{A}$ and each column of $\bm{B}$. In this way, the sparsity of $\bm{B}\bm{A}$ can be bounded by $s(\bm{A}_{i*})$ and $s(\bm{B}_{*i})$. We present the theoretical analysis in Proposition~\ref{prop:2} and the empirical results in Fig.~\ref{fig:bound}. Based on this finding, we can convert the optimization problem in Eqn.~\ref{eq:loss} as the following problem:
\begin{align}
\label{eq:final_loss}
    \notag&\min_{\bm{A},\bm{B}}\quad \mathcal{L}(\mathcal{D};\bm{W}^{o}+\bm{B}\bm{A})\\
    &\text{s.t.} \quad \frac{\|\bm{A}_{i*}\|_{0}}{d_{2}}\leq \tau, \frac{\|\bm{B}_{*i}\|_{0}}{d_{1}}\leq \tau, i=1,...,r.
\end{align}

\begin{proposition}
\label{prop:2}
The sparsity of $\bm{B}\bm{A}$ is greater or equal to $\max\{0,1+\sum_{i=1}^{r}(s(\bm{A}_{i*})+s(\bm{B}_{*i})-s(\bm{A}_{i*})s(\bm{B}_{*i}))-r\}$.
\end{proposition}

\begin{figure}[htb!] 
\centering        
\includegraphics[width=1\columnwidth]{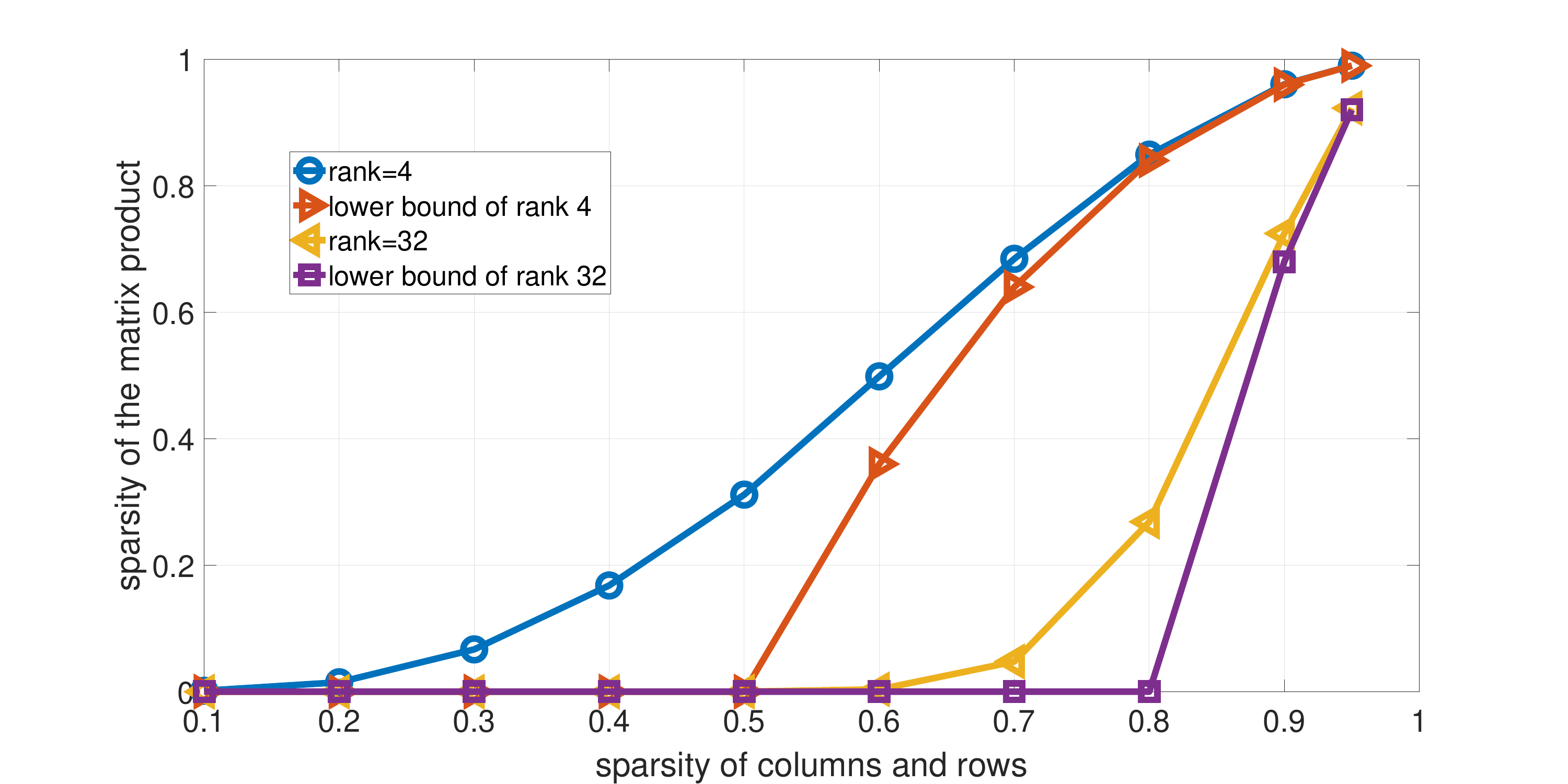} 
\caption{The sparsity of the product of matrix $\bm{B}$ and $\bm{A}$ with different column and row sparsity.} 
\label{fig:bound}
\end{figure}

\subsection{Optimization}
In this section, we present how to solve the optimization problem in Eqn.~\ref{eq:final_loss}. We prune each row of $\bm{A}$ and each column of $\bm{B}$ based on sensitivity iteratively. Specifically, we first conduct stochastic gradient decent with respective to $\bm{A}$ and $\bm{B}$, i.e.
\begin{align}
    \notag\tilde{\bm{A}}^{(t)}=\bm{A}^{(t)}-\nabla_{\bm{A}^{(t)}}\mathcal{L},\\
    \tilde{\bm{B}}^{(t)}=\bm{B}^{(t)}-\nabla_{\bm{B}^{(t)}}\mathcal{L}.
\end{align}
Then, we estimate the sensitivity-based importance scores based on Eqn.~\ref{eq:sensitivity}. Given the importance scores, the $\bm{A}$ and $\bm{B}$ are pruned following
\begin{align}
\notag\bm{A}^{(t+1)}_{i*}=\mathcal{T}_{A}(\tilde{\bm{A}}^{(t)}_{i*},\bar{I}^{(t)}(\bm{A}^{(t)}_{i*})),\\
\bm{B}^{(t+1)}_{*i}=\mathcal{T}_{B}(\tilde{\bm{B}}^{(t)}_{*i},\bar{I}^{(t)}(\bm{B}^{(t)}_{*i})),
\end{align}
where $i=1,2,...,r$, $\mathcal{T}_{A}$ is defined as
\begin{align}
    \notag&(\mathcal{T}_{A}(\tilde{\bm{A}}^{(t)}_{i*},\bar{I}^{(t)}(\bm{A}^{(t)}_{i*})))_{j}\\
    &=\left\{
\begin{aligned}
\notag\tilde{\bm{A}}^{(t)}_{ij} & , & \bar{I}^{(t)}(\bm{A}^{(t)}_{ij}) \text{ is top-}\tau^{(t)} \text{ in } \bar{I}^{(t)}(\bm{A}^{(t)}_{i*}), \\
0 & , & \text{otherwise},
\end{aligned}
\right.
\end{align}
and $\mathcal{T}_{B}$ is defined as
\begin{align}
    \notag&(\mathcal{T}_{B}(\tilde{\bm{B}}^{(t)}_{*i},\bar{I}^{(t)}(\bm{B}^{(t)}_{*i})))_{j}\\
    &=\left\{
\begin{aligned}
\notag\tilde{\bm{B}}^{(t)}_{ji} & , & \bar{I}^{(t)}(\bm{B}^{(t)}_{ji}) \text{ is top-}\tau^{(t)} \text{ in } \bar{I}^{(t)}(\bm{B}^{(t)}_{*i}), \\
0 & , & \text{otherwise}.
\end{aligned}
\right.
\end{align}
Here, $\tau^{(t)}$ is the budget of the percentage of remaining parameters at the $t$-iteration. To enable the optimization to be more stable, we decrease the number of $\tau^{(t)}$ gradually following the cubic strategy~\cite{li2023losparse}:
\begin{align}
\small
   \notag \tau^{(t)}=\left\{
\begin{aligned}
&1  , & 1\leq t\le t_{i}, \\
&\tau+(1-\tau)\left(1-\frac{t-t_{i}}{t_{f}-t_{i}}\right)^{3}  , & t_{i}\leq t\le t_{f},\\
&\tau, & t_{f}\leq t\leq T,
\end{aligned}
\right.
\end{align}
where $T$ is the number of total training iterations, and $t_{i}, t_{f}$ are hyper-parameters.

\section{Experiment}
In the experiments, we evaluate the proposed {\name} and answer the following questions: \textbf{RQ1})~How does the proposed {\name} benefit knowledge editing tasks? \textbf{RQ2})~How does {\name} perform compared to state-of-the-art PEFT methods on general tasks? \textbf{RQ3})~Does the proposed {\name} alleviate the model forgetting issue? \textbf{RQ4})~How does the performance change with varying amounts of training data?

\subsection{Datasets and Experiment Settings}
\paragraph{Datasets.} 
We conduct experiments on five different benchmarks: 
1)~\textbf{Knowledge Editing}, including WikiData$_{\text{recent}}$, WikiData$_{\text{counterfact}}$ \citep{cohen2024evaluating}, ZsRE \citep{yao2023editing}, and WikiBio \citep{hartvigsen2024aging}; 
2)~\textbf{Commonsense Reasoning}, including BoolQ \citep{clark2019boolq}, PIQA \citep{bisk2020piqa}, SIQA \citep{sap2019socialiqa}, HellaSwag \citep{zellers2019hellaswag}, WinoGrande \citep{sakaguchi2021winogrande}, ARC-e, ARC-c \citep{clark2018think}, and OBQA \citep{mihaylov2018can}; 
3)~\textbf{Arithmetic Reasoning}, including AQuA \citep{ling2017program}, GSM8K \citep{cobbe2021training}, MAWPS \citep{koncel2016mawps}, and SVAMP \citep{patel2021nlp}; 
4)~\textbf{Instruction Following} with Ultrafeedback~\citep{cui2023ultrafeedback} as training data and evaluation on Alpaca-Eval v1.0~\citep{li2023alpacaeval}; 
5)~\textbf{Natural Language Understanding} consists of eight datasets from the GLUE benchmark~\cite{wang2018glue}. More details about datasets, metrics, and hyper-parameters we use can be found in the Appendix.

\paragraph{Baselines.} 
Our baselines are constructed on a task basis.
In specific, on each task the proposed {\name} is compared with representative baselines from corresponding domain as listed below.
\begin{itemize}
\item 
On Knowledge Editing, we follow \citet{zhang2024comprehensive} and choose AdaLoRA \citep{zhang2023adaptive}, ROME and FT-L \citep{meng2022locating}, 
and MEMIT \citep{meng2022mass} as our baselines as they, same as us, do not require hard-to-access data or training additional models. 
In specific, AdaLoRA keeps unimportant weights in an LLM unchanged and achieves a highly efficient and precise PEFT.
ROME applies a causal-tracing to identify the layer wherein the knowledge is stored and then learns a rank-one update. 
FT-L, on the other hand, directly finetunes the layer identified by ROME. 
Recently, MEMIT extends ROME to a large-scale setting, where the edits can be made more efficiently. 

\item
On the other four tasks, we follow the setup from existing works \citep{hu2023llm,liu2024dora,wu2024reft} that evaluated a variety of representative PEFT methods including prefix tuning \citep{li2021prefix}, adapters \citep{houlsby2019parameter}, LoRA and its recent variants \citep{hu2021lora,zhang2023adaptive}, and ReFT \cite{wu2024reft}. Due to page limitation we refer the readers to \citet{hu2023llm,wu2024reft} and reference therein for more details. 
\end{itemize}

\subsection{Performance Comparison}
\paragraph{Knowledge Editing}

When performing knowledge editing, we introduce an additional norm constraint for low-rank matrices, as detailed in the Appendix. The results of knowledge editing are presented in Table~\ref{tab:know_edit}, addressing RQ1. From this table, we observe that the proposed {\name} outperforms all state-of-the-art baselines in terms of average performance, achieving the highest edit success rate while preserving the most knowledge that should not be updated. Moreover, {\name} demonstrates excellent generalization ability, as indicated by its high portability score which is a metric to measure if the edited
model can reason correctly about the updated knowledge.

\begin{table*}[h!]
\centering
\caption{Performance comparison of LLaMA-7b-chat against existing knowledge editing methods on four knowledge editing datasets. Results marked with "$\heartsuit$" are taken from \citet{zhang2024comprehensive}. "AVG" means the average of edit success, locality, portability, and fluency. Because fluency is not at the same magnitude as other metrics, we leverage "fluency/10" when computing AVG values.}
\label{tab:know_edit}
\begin{tabular}{@{}llcccccc@{}}
\toprule
Dataset                                & \multicolumn{1}{c}{Metric} & FT-L$^{\heartsuit}$  & AdaLoRA$^{\heartsuit}$ & ROME$^{\heartsuit}$ & MEMIT$^{\heartsuit}$ & \cellcolor{gray!20} \name \\ \midrule
\multirow{4}{*}{WikiData$_\text{recent}$}   & Edit Succ.$(\uparrow)$                 & 71.2   & 65.6    & 85.1 & 85.3  & \cellcolor{gray!20} \textbf{98.4}     \\
                                       & Locality$(\uparrow)$                   & 63.7  & 55.8    & 66.2 & 64.8  & \cellcolor{gray!20} \textbf{83.4}     \\
                                       & Portability$(\uparrow)$                & 48.7  & 47.2    & 37.5 & 37.9  & \cellcolor{gray!20} \textbf{54.3}     \\
                                       & Fluency$(\uparrow)$                    & 549    & 538     & 574  & 567   & \cellcolor{gray!20} \textbf{585}      \\
                                       &\cellcolor{gray!20} AVG$(\uparrow)$                         & \cellcolor{gray!20} 59.6  & \cellcolor{gray!20} 55.6    & \cellcolor{gray!20} 61.5  & \cellcolor{gray!20} 61.2  & \cellcolor{gray!20} \textbf{73.7}   \\
                                       \midrule
\multirow{4}{*}{WikiData$_\text{counterfact}$} & Edit Succ.$(\uparrow)$              & 51.1   & 72.1    & 83.2 & 83.4  & \cellcolor{gray!20} \textbf{99.4}     \\
                                       & Locality$(\uparrow)$                   & 62.5  & 66.8    & 65.4 & 63.7  & \cellcolor{gray!20} \textbf{90.9}     \\
                                       & Portability$(\uparrow)$                & 39.1  & 55.2    & 38.7 & 40.1  & \cellcolor{gray!20} \cellcolor{gray!20} \textbf{57.2}     \\
                                       & Fluency$(\uparrow)$                    & 545    & 554     & 579  & 569   & \cellcolor{gray!20} \textbf{592}      \\
                                       & \cellcolor{gray!20} AVG$(\uparrow)$                        & \cellcolor{gray!20} 51.8   & \cellcolor{gray!20} 62.4    & \cellcolor{gray!20} 61.3 & \cellcolor{gray!20} 61.0  & \cellcolor{gray!20} \textbf{76.7}    \\
                                       \midrule
\multirow{4}{*}{ZsRE}                  & Edit Succ.$(\uparrow)$                 & 51.1   & 72.1    & 83.2 & 83.4  & \cellcolor{gray!20} \textbf{100}      \\
                                       & Locality$(\uparrow)$                   & 62.5  & 66.8    & 65.4 & 63.7  & \cellcolor{gray!20} \textbf{92.5}     \\
                                       & Portability$(\uparrow)$                & 39.1  & \textbf{55.2}    & 38.7 & 40.1  & \cellcolor{gray!20} 50.9     \\
                                       & Fluency$(\uparrow)$                    & 545    & 554     & \textbf{579}  & 569   & \cellcolor{gray!20} 574      \\ 
                                       & \cellcolor{gray!20} AVG$(\uparrow)$                        & \cellcolor{gray!20} 54.6   & \cellcolor{gray!20} 62.1    & \cellcolor{gray!20} 58.2         & \cellcolor{gray!20} 54.0 &\cellcolor{gray!20} \textbf{75.2} \\
                                       \midrule
\multirow{3}{*}{WikiBio}               & Edit Succ.$(\uparrow)$                 & 66.3   & 97.0    & 95.1 & 94.3  & \cellcolor{gray!20} \textbf{99.5}     \\
                                       & Locality$(\uparrow)$                   & 60.1  & 57.9    & 47.0 & 51.6  & \cellcolor{gray!20} \textbf{92.5}     \\
                                       & Fluency$(\uparrow)$                    & 604   & 616     & 617  & 617   & \cellcolor{gray!20} \textbf{620}      \\ 
                                       & \cellcolor{gray!20} AVG$(\uparrow)$                        & \cellcolor{gray!20} 62.3  & \cellcolor{gray!20} 72.2    & \cellcolor{gray!20} 67.9 & \cellcolor{gray!20} 69.2  & \cellcolor{gray!20} \textbf{84.6}     \\
                                       \bottomrule
\end{tabular}%
\end{table*}

\paragraph{Commonsense Reasoning}
In this section, we present experiments on eight commonsense reasoning datasets to address RQ2, as shown in Table~\ref{tab:commonsense}. The table indicates that the proposed {\name} again outperforms all state-of-the-art parameter-efficient fine-tuning methods on average. Among the eight datasets, {\name} ranks the first in five cases. Remarkably, its parameter numbers are the same as that of LoReFT, significantly smaller than PrefT, Adapter, LoRA, and DoRA. Yet, {\name} still achieves higher accuracy on the commonsense reasoning datasets. This clearly demonstrates {\name}'s effectiveness of fine-tuning the most crucial parameters of LLaMA for commonsense reasoning tasks.

\begin{table*}[h!]

\centering
\caption{Accuracy comparison of LLaMA-7B against PEFT baselines on eight commonsense reasoning datasets. Results marked with "$\heartsuit$" are taken from \citet{liu2024dora}. "AVG" means the average accuracy of all datasets. For {\name}, Params (\%) is calculated by dividing the number of final low-rank matrices parameters by the number of parameters of the base LMs~(number of low-rank matrix parameters times sparsity).}
\label{tab:commonsense}
\resizebox{0.955\textwidth}{!}{%
\begin{tabular}{@{}cccccccccc>{\columncolor{gray!20}}c@{}}
\toprule
\multirow{2}{*}{PEFT} & \multirow{2}{*}{Params (\%)} & \multicolumn{9}{c}{Accuracy $(\uparrow)$}                                         \\ \cmidrule(l){3-11} 
                      &                              & BoolQ & PIQA & SIQA & HellaS. & WinoG. & ARC-e & ARC-c & OBQA &  \textbf{AVG}  \\ \midrule 
PrefT$^{\heartsuit}$                 & 0.11\%                      & 64.3  & 76.8 & 73.9 & 42.1    & 72.1   & 72.9  & 54.0  & 60.6 & 64.6 \\
Adapter$^{\text{S}}$$^{\heartsuit}$               & 0.99\%                      & 63.0  & 79.2 & 76.3 & 67.9    & 75.7   & 74.5  & 57.1  & 72.4 & \cellcolor{gray!20} 70.8 \\
Adapter$^{\text{P}}$$^{\heartsuit}$                  & 3.54\%                      & 67.9  & 76.4 & 78.8 & 69.8    & 78.9   & 73.7  & 57.3  & 75.2 & \cellcolor{gray!20} 72.3 \\
LoRA$^{\heartsuit}$                  & 0.83\%                      & 68.9  & 80.7 & 77.4 & 78.1    & 78.8   & 77.8  & 61.3  & 74.8 &\cellcolor{gray!20} 74.7 \\
DoRA (half)$^{\heartsuit}$           & 0.43\%                      & 70.0  & 82.6 & 79.7 & 83.2    & 80.6   & 80.6  & 65.4  & 77.6 &\cellcolor{gray!20} 77.5 \\
DoRA$^{\heartsuit}$                  & 0.84\%                      & 68.5  & 82.9 & 79.6 & 84.8    & 80.8   & 81.4  & 65.8  & 81.0 &\cellcolor{gray!20} 78.1 \\
LoReFT$^{\heartsuit}$                & 0.03\%                      & 69.3  & 84.4 & \textbf{80.3} & \textbf{93.1}    & \textbf{84.2}   & 83.2  & 68.2  & 78.9 &\cellcolor{gray!20} 80.2 \\ \midrule
\cellcolor{gray!20} \name              & \cellcolor{gray!20} 0.03\%                      & \cellcolor{gray!20} \textbf{71.0}  & \cellcolor{gray!20} \textbf{84.9} & \cellcolor{gray!20} 75.5 & \cellcolor{gray!20} 92.6    & \cellcolor{gray!20} 82.6   & \cellcolor{gray!20} \textbf{84.6}  & \cellcolor{gray!20} \textbf{70.0}  & \cellcolor{gray!20} \textbf{84.2} &\cellcolor{gray!20} \textbf{80.7} \\ \bottomrule
\end{tabular}%
}
\end{table*}
\paragraph{Arithmetic Reasoning}

In this section, we present experiments on four arithmetic reasoning datasets to address RQ2, with results shown in Table~\ref{tab:math}. The table indicates that LoRA achieves the highest average accuracy across the four datasets. However, the proposed {\name} performs comparably, retaining 97\% of LoRA's accuracy while updating only 22 times less parameters compared with LoRA. Additionally, compared to LoReFT, {\name} updates a similar number of parameters while achieving approximately a 6.3\% performance improvement.

\begin{table*}[h!]
\centering
\caption{Accuracy comparison of LLaMA-7B against PEFT baselines on four arithmetic reasoning datasets. Results marked with "$\heartsuit$" are taken from \citet{hu2023llm}. "AVG" means the average accuracy of all datasets.}
\label{tab:math}
\begin{tabular}{@{}ccccccc@{}}
\toprule
\multirow{2}{*}{PEFT} & \multirow{2}{*}{Params (\%)} & \multicolumn{5}{c}{Accuracy $(\uparrow)$}                                                  \\ \cmidrule(l){3-7} 
                      &                              & AQuA          & GSM8K         & MAWPS         & SVAMP         & \cellcolor{gray!20} AVG           \\ \midrule
PrefT$^{\heartsuit}$                 & 0.11\%                      & 14.2          & 24.4          & 63.4          & 38.1          & \cellcolor{gray!20}35.0          \\
Adapter$^{\text{S}}$$^{\heartsuit}$               & 0.99\%                      & 15.0          & 33.3          & 77.7          & \textbf{52.3} & \cellcolor{gray!20}44.6          \\
Adapter$^{\text{P}}$$^{\heartsuit}$               & 3.54\%                      & 18.1          & 35.3          & \textbf{82.4} & 49.6          & \cellcolor{gray!20}46.4          \\
LoRA$^{\heartsuit}$                  & 0.83\%                      & 18.9          & \textbf{37.5} & 79.0          & 52.1          & \cellcolor{gray!20}\textbf{46.9} \\
LoReFT$^{\heartsuit}$                & 0.03\%                      & 21.4          & 26.0          & 76.2          & 46.8          & \cellcolor{gray!20}42.6          \\ \midrule
\rowcolor{gray!20} 
\name              & 0.03\%                      & \textbf{26.0} & 33.0          & 79.8          & 44.7          & 45.9          \\ \bottomrule
\end{tabular}%
\end{table*}

\paragraph{Instruction Following}
In this section, we compare the proposed {\name} with state-of-the-art baselines on the instruction-following task. To ensure fair comparisons, we use the same prompt templates from \citet{taori2023alpaca}. The model performance is shown in Table~\ref{tab:instruction}. Based on the table, it can be observed that the proposed {\name} outperforms all baseline methods while updating the fewest parameters. Additionally, for the instruction-following task, we find that significantly fewer parameters need to be updated compared to commonsense reasoning and arithmetic reasoning tasks. This suggests that fewer parameters are related to the instruction-following ability in the large language model.
\begin{table}[h!]
\centering
\caption{Performance comparison of LLaMA-2 7B on instruction tuning task on Alpaca-Eval v1.0. We compute the win-rate against text-davinci-003 using GPT-4 as the annotator. 
Results marked with "$\heartsuit$" are taken from \citet{wu2024reft}.}
\label{tab:instruction}
\resizebox{\columnwidth}{!}{%
\begin{tabular}{@{}lrc@{}}
\toprule
Model \& PEFT          & Params (\%) & Win-rate $(\uparrow)$ \\ \midrule
GPT-3.5 Turbo 1106$^{\heartsuit}$  & -           & 86.30    \\ \midrule
Llama-2 Chat 13B$^{\heartsuit}$    & -           & 81.10    \\
Llama-2 Chat 7B$^{\heartsuit}$     & -           & 71.40    \\
Llama-2 7B \& FT$^{\heartsuit}$       & 100\%       & 80.93    \\
Llama-2 7B \& LoRA$^{\heartsuit}$     & 0.1245\%    & 81.48    \\
Llama-2 7B \& RED$^{\heartsuit}$      & 0.0039\%    & 81.69    \\
Llama-2 7B \& LoReFT$^{\heartsuit}$   & 0.0039\%    & 85.60    \\ \midrule
\cellcolor{gray!20} Llama-2 7B \& \name & \cellcolor{gray!20} 0.0037\%    & \cellcolor{gray!20} \textbf{85.77}    \\ \bottomrule
\end{tabular}%
}
\end{table}

\paragraph{Natural Language Understanding}
We conduct experiments on the GLUE to answer RQ2. We show the model performance in Table~\ref{tab:nlu}. According to the table, the proposed {\name} outperforms the state-of-the-art baselines significantly. The best baseline LoRA achieves 88.1 average accuracy but the proposed {\name} reaches about 89.0 accuracy on the eight datasets averagely. On RTE dataset, the proposed {\name} even achieves 3.4\% performance improvement. Compared to fully fine-tuning, the proposed {\name} also achieves better performance. The potential reason may be that {\name} only updates very few parameters and prevents overfitting on natural language understanding tasks. It demonstrates that the proposed {\name} not only can be applied to decoder-only models but also can be applied to encoder-only language models.
\begin{table*}[h!]
\centering
\caption{Accuracy comparison of RoBERTa-large against PEFT baselines on the GLUE benchmark. 
Results marked with "$\heartsuit$" are taken from \citet{wu2023eva}. 
"AVG" means the average accuracy of all datasets.}
\label{tab:nlu}
\resizebox{\textwidth}{!}{%
\begin{tabular}{@{}ccccccccccc@{}}
\toprule
PEFT        & Params (\%) & RTE           & MRPC          & QQP           & STS-b         & QNLI          & CoLA          & SST2          & MNLI          & \cellcolor{gray!20} AVG           \\ \midrule
FT$^{\heartsuit}$          & 100\%       & 85.8          & 91.7          & 91.5          & 92.6          & 93.8          & 68.2          & 96.0          & 88.8          & \cellcolor{gray!20} 88.6          \\ \midrule
Adapter$^{\heartsuit}$     & 0.254\%     & 85.3          & \textbf{90.5} & \textbf{91.4} & 91.5          & 94.6          & 65.4          & 95.2          & 90.1          & \cellcolor{gray!20} 88.0          \\
LoRA$^{\heartsuit}$        & 0.225\%     & 86.3          & 89.8          & 90.7          & 91.7          & \textbf{94.7} & 65.5          & 96.0          & 90.2          & \cellcolor{gray!20} 88.1          \\
Adapter$^\text{FNN}$$^{\heartsuit}$ & 0.225\%     & 84.8          & 90.5          & 91.3          & 90.2          & 94.3          & 64.4          & 96.1          & 90.3          & \cellcolor{gray!20} 87.7          \\
RED$^{\heartsuit}$         & 0.014\%     & 86.2          & 90.3          & 88.8          & 91.3          & 93.5          & 68.1          & 96.0          & 89.5          & \cellcolor{gray!20} 88.0          \\
LoReFT$^{\heartsuit}$      & 0.014\%     & 86.2          & 90.1          & 88.5          & 91.6          & 94.1          & 68.0          & \textbf{96.2} & 89.2          & \cellcolor{gray!20} 88.0          \\ \midrule
\rowcolor{gray!20}
\name  & 0.015\%     & \textbf{89.2} & 90.2          & 91.1          & \textbf{92.0} & \textbf{94.7} & \textbf{69.2} & 95.2          & \textbf{90.5} & \textbf{89.0} \\ \bottomrule
\end{tabular}%
}
\end{table*}

\subsection{Forgetting Test}
In this section, we study if a fine-tuned model forgets knowledge learned from the pre-training stage to answer RQ3. To make fair comparisons, we evaluate LoRA and {\name} after fine-tuning on Commonsense170K, Ultrafeedback, and Math10K in a zero-shot setting and using the same prompt templates. We report the experiment results in Table~\ref{tab:forget}. According to the table, we can find that compared to LoRA, the {\name} forgets less knowledge after fine-tuning. For example, after fine-tuning on the Commonsense170K dataset, LoRA leads to a significant performance drop on TriviaQA and MMLU. However, the proposed {\name} still preserves over 90\% performance of LLaMA-2. Besides, we can also find that both LoRA and {\name} achieve good performance on ARC-c dataset. It may indicate that fine-tuning large language models on Commonsense170K, Ultrafeedback, or Math10K may not make them forget much general knowledge.

\begin{table*}[h!]
    \centering
    \caption{Accuracy of fine-tuned models on TriviaQA~(knowledge reasoning), MMLU~(general knowledge), and ARC-c~(commonsense reasoning) dataset. "AVG" is the average accuracy of Humanities, Social Sciences, STEM, and Other fields on MMLU. The evaluation is conducted with Lm-Evaluation-Harness \citep{gao20232eval}.}
    \label{tab:forget}
    \resizebox{\textwidth}{!}{%
    \begin{tabular}{@{}lccccccc@{}}
    \toprule
    \multirow{2}{*}{}     & \multirow{2}{*}{TriviaQA} & \multicolumn{5}{c}{MMLU}                                                                                                                              & \multirow{2}{*}{ARC-c} \\ \cmidrule(lr){3-7}
                          &                           & \multicolumn{1}{c}{Humanities} & \multicolumn{1}{c}{Social Sciences} & \multicolumn{1}{c}{STEM} & \multicolumn{1}{c}{Other} & \multicolumn{1}{c}{AVG} &                        \\ \midrule
    LLaMA 7B & 48.6 & 29.9 & 29.4 & 26.3 & 33.4 & 29.8 &  41.7\\ \hdashline 
    After Commonsense170K &                           &                                &                                     &                          &                           &                         &                        \\
    LoRA                  & 9.0                       & 24.4                           & 21.9                                & 21.5                     & 24.0                      & 23.1                    & -                   \\
    \cellcolor{gray!20} {\name}              & \cellcolor{gray!20} \textbf{47.8}                      & \cellcolor{gray!20} \textbf{36.8}                           & \cellcolor{gray!20} \textbf{42.7}                                & \cellcolor{gray!20} \textbf{31.4}                     & \cellcolor{gray!20} \textbf{42.3}                      & \cellcolor{gray!20} \textbf{38.1}                   & \cellcolor{gray!20} -                   \\ \hdashline
    After Math10K         &                           &                                &                                     &                          &                           &                         &                        \\
    LoRA                  & 30.5                      & 31.1                           & 34.4                                & 30.5                     & 35.7                      & 32.7                    & \textbf{42.2}                   \\
    \cellcolor{gray!20} {\name}              & \cellcolor{gray!20} \textbf{51.3}                      & \cellcolor{gray!20} \textbf{37.9}                           & \cellcolor{gray!20} \textbf{43.0}                                & \cellcolor{gray!20} \textbf{32.1}                     & \cellcolor{gray!20} \textbf{43.9 }                     & \cellcolor{gray!20} \textbf{39.0}                    & \cellcolor{gray!20} 41.9                   \\ \midrule 
    LLaMA-2 7B            & 52.5                      & 38.9                           & 46.1                                & 34.3                     & 47.1                      & 41.2                    & 43.4                   \\ \hdashline
    After Ultrafeedback   &                           &                                &                                     &                          &                           &                         &                        \\
    LoRA                  & 23.5                      & 41.3                           & 49.4                                & 43.0                     & 49.3                      & 43.0                    & 41.2                   \\
    \cellcolor{gray!20} {\name}              & \cellcolor{gray!20} \textbf{30.1 }                     & \cellcolor{gray!20} \textbf{42.1}                           & \cellcolor{gray!20} \textbf{51.5}                                & \cellcolor{gray!20} \textbf{44.9}                     & \cellcolor{gray!20} \textbf{52.0}                      & \cellcolor{gray!20} \textbf{44.9 }                   & \cellcolor{gray!20} \textbf{44.4}                   \\ \bottomrule

    \end{tabular}%
    }
    \end{table*}

\subsection{Sensitivity w.r.t. Training Data Size}
In this section, we study how the model performance changes with different amounts of training data. We show the experiment results in Fig.~\ref{fig:numdata}. Based on the figure, we can find that with the decreasing amounts of training data, the performance gap between LoRA and {\name} is becoming smaller. When using only 12.5\% Math10K data as the training data to fine-tune the LLaMA 7B, {\name} even outperforms LoRA on GSM8K. In conclusion, the proposed {\name} shows more superiority on small data scenarios.
\begin{figure}[htb!] 
\centering        
\includegraphics[width=1.0\columnwidth]{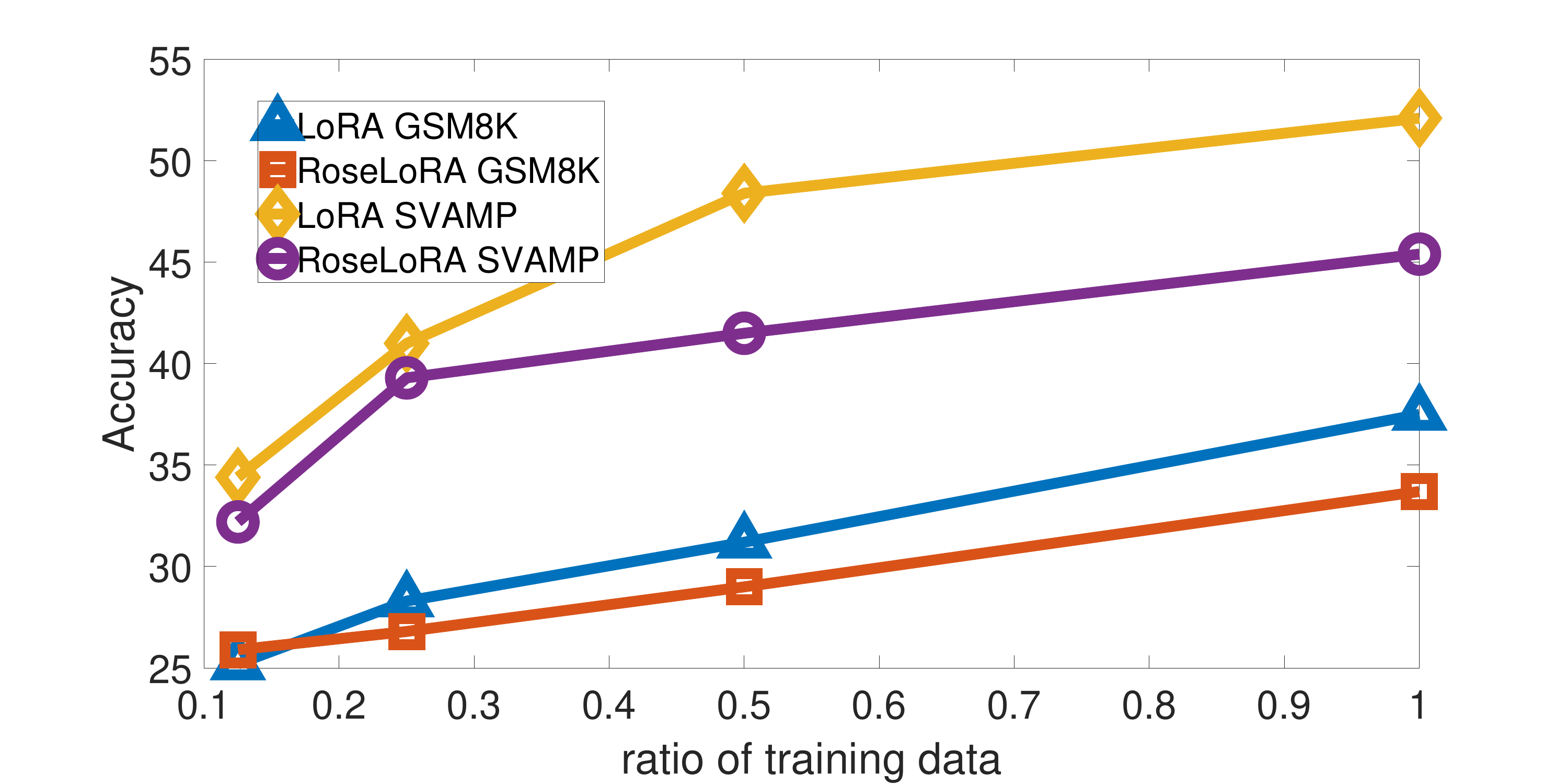} 
\caption{Accuracy of LoRA and {\name} with different amount of Math10K training data on GSM8K and SVAMP.} 
\label{fig:numdata}
\vspace{-0.1in}
\end{figure}

\section{Conclusion}
In this paper, we address the limitations of existing parameter-efficient fine-tuning (PEFT) methods, particularly the LoRA family, in handling tasks requiring selective knowledge updates while still being effective for other general NLP tasks. We introduced a novel method, \textbf{r}ow and c\textbf{o}lumn-wise spar\textbf{se} \textbf{lo}w-\textbf{r}ank \textbf{a}daptation ({\name}), which selectively updates the most important parameters for specific tasks, maintaining efficiency while minimizing unnecessary changes to the pre-trained model's knowledge. {\name} applies a row and column-wise sparsity constraint to the product of low-rank matrices, ensuring efficient updates without modifying the model architecture. Our theoretical analysis lower bounds the sparsity of product matrices that affect model's knowledge, and our sensitivity-based importance scoring effectively fulfilled the sparsity constraints. Through extensive experiments on five benchmarks encompassing over twenty datasets, {\name} demonstrated superior performance on both general-purposed fine-tuning and knowledge editing tasks compared to existing methods. This highlights its potential as a robust and efficient fine-tuning solution for a wide range of NLP applications.

\section*{Limitations}
The proposed {\name} framework introduces a hyper-parameter $\beta$ to smooth the sensitivity estimation, which might require additional effort to tune. Fortunately, we observe that the model performance is not sensitive to the hyper-parameter and we set it to a fixed value to achieve good performance in this paper.

\section*{Acknowledgement}
This work is supported in part by the US National Science Foundation under grant NSF IIS-1747614 and NSF IIS-2141037. Any opinions, findings, and conclusions or recommendations expressed in this material are those of the author(s) and do not
necessarily reflect the views of the National Science Foundation.

\bibliography{ref}

\newpage
\appendix
\setcounter{lemma}{0} 
\setcounter{proposition}{0} 
\section{Proof of Proposition~1}
\begin{lemma}
\label{lemma:1}
    For $\bm{a}\in\mathbb{R}^{1\times d_{2}}$ and $\bm{b}\in\mathbb{R}^{d_{1}\times 1}$, where the sparsity of them is $s(\bm{a})=s_{a}$ and $s(\bm{b})=s_{b}$ respectively, we have $s(\bm{b}\bm{a})=s_a+s_b-s_{a}s_{b}$.
\end{lemma}
\begin{proof}
Define the number of zero values in a vector or matrix as $z(\cdot)$. Consider the $i$-th row of $\bm{b}\bm{a}$, i.e. $\bm{b}_{i}\bm{a}$. If $\bm{b}_{i}=0$, then $\bm{b}_{i}\bm{a}=\bm{0}$. If $\bm{b}_{i}\neq0$, then the number of zeros depends on the number of zeros of $\bm{a}$. Therefore, we have
\begin{align}
    z(\bm{b}_{i}\bm{a})=\left\{
\begin{aligned}
d_{2} & , & \bm{b}_{i}=0, \\
s_{a}d_{2} & , & \bm{b}_{i}\neq0.
\end{aligned}
\right.
\end{align}
Then we have
\begin{align}
    \notag z(\bm{b}\bm{a})=&\sum_{i=1}^{d_{1}}z(\bm{b}_{i}\bm{a})\\
    \notag=&d_{2}s_{b}d_{1}+s_{a}d_{1}d_{2}(1-s_{b})\\
    =&d_{1}d_{2}(s_{a}+s_{b}-s_{a}s_{b}).
\end{align}
So the sparsity of $\bm{b}\bm{a}$ is
\begin{align}
    \notag s(\bm{b}\bm{a})&=\frac{d_{1}d_{2}(s_{a}+s_{b}-s_{a}s_{b})}{d_{1}d_{2}}\\
    &=s_{a}+s_{b}-s_{a}s_{b}.
\end{align}
\end{proof}

\begin{proposition}
The sparsity of $\bm{B}\bm{A}$ is greater or equal to $\max\{0,1+\sum_{i=1}^{r}(s(\bm{A}_{i*})+s(\bm{B}_{*i})-s(\bm{A}_{i*})s(\bm{B}_{*i}))-r\}$.
\end{proposition}
\begin{proof}
First, we have
\begin{align}
\notag (\bm{B}\bm{A})_{ij}&=\sum_{k=1}^{r}\bm{B}_{ik}\bm{A}_{kj}\\
&= \sum_{k=1}^{r}(\bm{B}_{*k}\bm{A}_{k*})_{ij}.
\end{align}
Consider the worst case: the positions of nonzero value of $\{\bm{B}_{*k}\bm{A}_{k*}\}$ does not have any overlapping, we at least have $\max\{0,d_{1}d_{2}-\sum_{i=1}^{r}(1-s(\bm{B}_{*i}\bm{A}_{i*}))d_{1}d_{2}\}$ zero values. 

Therefore, based on Lemma~\ref{lemma:1} the sparsity of $\bm{B}\bm{A}$ satisfies
\begin{align}
    \notag & s(\bm{B}\bm{A}) \\
    \geq&\frac{\max\{0,d_{1}d_{2}-\sum_{i=1}^{r}(1-s(\bm{B}_{*i}\bm{A}_{i*}))d_{1}d_{2}\}}{d_{1}d_{2}} \notag\\
    \notag =&\max\{0,1+\sum_{i=1}^{r}s(\bm{B}_{*i}\bm{A}_{i*})-r\}\\
    \notag=&\max \bigg\{0,1+\sum_{i=1}^{r}\big(s(\bm{A}_{i*})+s(\bm{B}_{*i})\bigg.\\
    &\big.-s(\bm{A}_{i*})s(\bm{B}_{*i})\big)-r\bigg\}.
\end{align}
\end{proof}

\begin{table*}[h!]
\centering
\caption{Hyper-parameters used in knowledge editing, commonsense reasoning and arithmetic reasoning.}
\label{tab:hyper}
\resizebox{\textwidth}{!}{%
\begin{tabular}{@{}cccccccc@{}}
\toprule
Dataset              & lr   & Rank & Batch size & Sparsity & $\beta$                 & $\alpha$ & Target modules                         \\ \midrule
WikiData recent      & 2e-4 & 4    & 1          & 0.95     & \multirow{7}{*}{0.8} & 3e-3  & "up\_proj", "down\_proj", "gate\_proj" \\ 
WikiData counterfact & 2e-4 & 4    & 1          & 0.95     &                      & 3e-3  & "up\_proj", "down\_proj", "gate\_proj" \\
ZsRE                 & 2e-4 & 4    & 1          & 0.95     &                      & 3e-3  & "up\_proj", "down\_proj", "gate\_proj" \\
WikiBio              & 2e-4 & 4    & 1          & 0.95     &                      & 3e-3  & "up\_proj", "down\_proj", "gate\_proj" \\
Commonsense170K      & 2e-4 & 32   & 8          & 0.865    &                      & -     & "q\_proj","v\_proj"                    \\
Math10K              & 3e-4 & 32   & 32         & 0.865    &                      & -     & "q\_proj","v\_proj"                    \\
Instruction tuning   & 3e-4 & 32   & 32         & 0.85     &                      & -     & "q\_proj","v\_proj"                    \\ \bottomrule
\end{tabular}%
}
\end{table*}

\begin{table*}[h!]
\centering
\caption{Hyper-parameters and metrics used in GLUE benchmark.}
\label{tab:glue_data}
\begin{tabular}{@{}cccccccc@{}}
\toprule
Dataset & Metric        & lr   & Rank               & Batch size & Sparsity              & $\beta$                 & Target modules                                                                                                        \\ \midrule
CoLA    & Matthews corr & 2e-4 & \multirow{8}{*}{6} & 16         & \multirow{8}{*}{0.95} & \multirow{8}{*}{0.8} & \multirow{8}{*}{\begin{tabular}[c]{@{}c@{}}"query",\\"key",\\"value",\\ "output.dense",\\"intermediate.dense"\end{tabular}} \\
SST-2   & Accuracy      & 2e-4 &                    & 32         &                       &                      &                                                                                                                       \\
MRPC    & Accuracy      & 2e-4 &                    & 32         &                       &                      &                                                                                                                       \\
QQP     & Accuracy      & 1e-4 &                    & 32         &                       &                      &                                                                                                                       \\
STS-B   & Pearson corr  & 2e-4 &                    & 32         &                       &                      &                                                                                                                       \\
MNLI    & Accuracy      & 2e-4 &                    & 32         &                       &                      &                                                                                                                       \\
QNLI    & Accuracy      & 2e-4 &                    & 32         &                       &                      &                                                                                                                       \\
RTE     & Accuracy      & 6e-4 &                    & 32         &                       &                      &                                                                                                                       \\ \bottomrule
\end{tabular}%
\end{table*}

\section{Datasets, Metrics and Hyper-parameters}
We conduct experiments on five different benchmarks:
\begin{itemize}[leftmargin=1em]
    \item Knowledge editing consists of four datasets, including WikiData$_{\text{recent}}$, WikiData$_{\text{counterfact}}$ \citep{cohen2024evaluating}, ZsRE \citep{yao2023editing}, and WikiBio \citep{hartvigsen2024aging}. 
    For the knowledge editing tasks, the model should memorize new knowledge while preserving knowledge which does not need to update. Following \citet{zhang2024comprehensive}, we use four metrics to evaluate the editing performance: 1)~\textbf{Edit Success}, which estimates the accuracy with respect to both the knowledge needed to be updated and the similar expressions of the knowledge, 2)~\textbf{Locality}, which shows if the post-edited model keeps its original answer on the locality set, 3)~\textbf{Portability}, which is to measure if the post-edited model can reason correctly about the updated knowledge, and 4)~\textbf{Fluency}, which measures the model’s generation ability after editing via calculating the weighted average of bi-gram and tri-gram entropies.
    
    \item Commonsense reasoning contains of eight datasets, 
    including BoolQ \citep{clark2019boolq}, PIQA \citep{bisk2020piqa}, SIQA \citep{sap2019socialiqa}, HellaSwag \citep{zellers2019hellaswag}, WinoGrande \citep{sakaguchi2021winogrande}, ARC-e, ARC-c \citep{clark2018think}, and OBQA \citep{mihaylov2018can}. 
    These tasks are multiple choice problems. Following \citet{hu2023llm,wu2024reft}, 
    we fine-tune the LLM on a combined training dataset named Commonsense170K of these tasks and evaluate the Accuracy on individual test sets. 
    
    \item Arithmetic reasoning consists of four math reasoning datasets:
    AQuA \citep{ling2017program}, GSM8K \citep{cobbe2021training}, MAWPS \citep{koncel2016mawps}, and SVAMP \citep{patel2021nlp}. 
    Models need to generate correct answers and we use Accuracy as the evaluation metric following \citet{hu2023llm} as well. 
    Again, we replicate the setup in \citet{wu2024reft} and fine-tune the models on the combined training data named Math10K of the four tasks. 
    
    \item Instruction-following measures if the model can follow human instructions. 
    Same as before, we follow \citet{hu2023llm, wu2024reft} and use Ultrafeedback \citep{cui2023ultrafeedback} as the training data, 
    and evaluate the model performance by Alpaca-Eval v1.0 \citep{li2023alpacaeval}.
    
    \item Natural language understanding consists of eight datasets from the GLUE benchmark \citep{wang2018glue}. 
    We adopt the evaluation metrics and setups from \citet{wu2023eva}.
\end{itemize}
We show the hyper-parameters we use in Table~\ref{tab:glue_data} and Table~\ref{tab:hyper}. We conduct experiments based on libraries LLM-Adapters$\footnote{https://github.com/AGI-Edgerunners/LLM-Adapters}$, EasyEdit$\footnote{https://github.com/zjunlp/EasyEdit}$, and lm-evaluation-harness$\footnote{https://github.com/EleutherAI/lm-evaluation-harness}$.

\section{Implementation of Knowledge Editing}
To enable the minimal modification of the LLM, following~\cite{zhang2024comprehensive}, we add one $\ell_{2}$ norm on the low-rank matrices:
\begin{align}
    \notag&\min_{\bm{A},\bm{B}}\quad \mathcal{L}(\mathcal{D};\bm{W}^{o}+\bm{B}\bm{A})\\
    \notag&\text{s.t.} \quad \frac{\|\bm{A}_{i*}\|_{0}}{d_{2}}\leq \tau, \frac{\|\bm{B}_{*i}\|_{0}}{d_{1}}\leq \tau, i=1,...,r,\\
    &\quad \quad \quad \|\bm{A}\|_{F}^{2}\leq \alpha,\|\bm{B}\|_{F}^{2}\leq\alpha,
\end{align}
where $\alpha$ is a hyper-parameter. In each step, after pruning $\bm{A}$ and $\bm{B}$, we clip them to make them satisfy the $\ell_{2}$ norm constraint.
\end{document}